\documentclass[11pt]{article}
\usepackage{fullpage}
\usepackage{url}
\usepackage{subcaption}

\usepackage{amsmath}
\usepackage{amsthm}
\usepackage{amsfonts}
\usepackage{amssymb}
\usepackage{mathtools}
\usepackage{xcolor}
\usepackage{enumitem}
\usepackage{booktabs}

\usepackage{algorithm}
\usepackage{algorithmic}

\newtheorem{theorem}{Theorem}
\newtheorem{lemma}{Lemma}
\newtheorem{corollary}{Corollary}
\newtheorem{assumption}{Assumption}
\newtheorem{definition}{Definition}

\newtheorem{proposition}{Proposition}

\usepackage{caption}
\usepackage{booktabs}
\usepackage{subcaption}
\usepackage{enumitem}
\usepackage{natbib}

\makeatletter
\DeclareFontEncoding{LS1}{}{}
\DeclareFontSubstitution{LS1}{stix}{m}{n}
\DeclareMathAlphabet{\mathscr}{LS1}{stixscr}{m}{n}
\makeatother

\usepackage{tikz}
\usepackage{tikz-qtree}
\usepackage{mathtools}
\usetikzlibrary{calc}
\usetikzlibrary{decorations.pathreplacing,decorations.markings}
\tikzstyle{dot}=[circle,fill,black,inner sep=1pt]

\usepackage{color}
\usepackage{xspace}

\title{Offline Constrained RLHF with Multiple Preference Oracles}

\author{Brenden Latham and Mehrdad Moharrami\thanks{ 
            B.\ Latham is a student in the Computer Science Department, University of Iowa, Iowa City, IA, USA, \texttt{brenden-latham@uiowa.edu}.
		M.\ Moharrami is with Computer Science Department, University of Iowa, Iowa City, IA, USA, \texttt{mehrdad-moharrami@uiowa.edu}.
}}
\date{\today}

\begin{document}
	
\pgfdeclarelayer{background}
\pgfdeclarelayer{foreground}
\pgfsetlayers{background,main,foreground}

\maketitle

\begin{abstract}
We study offline constrained reinforcement learning from human feedback with multiple preference oracles. Motivated by applications that trade off performance with safety or fairness, we aim to maximize target population utility subject to a minimum protected group welfare constraint. From pairwise comparisons collected under a reference policy, we estimate oracle-specific rewards via maximum likelihood and analyze how statistical uncertainty propagates through the dual program. We cast the constrained objective as a KL-regularized Lagrangian whose primal optimizer is a Gibbs policy, reducing learning to a convex dual problem. We propose a dual-only algorithm with high-probability constraint violation bounds and establish finite-sample performance guarantees for offline constrained preference learning. Finally, we extend our theoretical analysis to accommodate multiple constraints and general $f$-divergence regularization.
\end{abstract}

\section{Introduction}
\begin{table*}[t]
    \centering
    \caption{Comparison of RLHF methods.}
    \label{tab:comparison}
    \resizebox{\textwidth}{!}{%
    \small
    \begin{tabular}{@{}lllll@{}}
        \toprule
        \textbf{Method} & \textbf{Objective} & \textbf{Setting} & \textbf{Guarantees} & \textbf{Constraint Type} \\
        \midrule
        \cite{chakraborty2024maxmin} & Max-Min Welfare & Online & Asymptotic/None & Unconstrained \\
        \cite{ramesh2024group} & Robust Utility & Offline & None & Robustness Penalty \\
        \cite{ge2024axioms} & Social Choice & Theoretical & Axiomatic & N/A \\
        \cite{dai2024safe} & Lagrangian PPO & Online & Empirical & Expected safety constraint \\
        \cite{zhou2024beyond} & Scalarization & Online & Empirical & Pareto \\
        \cite{wachi2024stepwise} & Constrained Alignment & Offline & Statistical optimality \& violation & Expected reward \\
        \midrule
        \textbf{Ours} & \textbf{Constrained RL} & \textbf{Offline} & \textbf{Finite-sample} & \textbf{Expected-reward constraint} \\
        \bottomrule
    \end{tabular}}
\end{table*}

\emph{Reinforcement Learning from Human Feedback} (RLHF) has quickly become the dominant practical strategy for aligning powerful AI systems with human values when direct reward engineering is infeasible or brittle. Early preference-based RL works established that effective control policies can be recovered from comparative human judgments alone~\citep{jain2013learning,fekete2014preference,daniel2015active,wirth2016model}, and subsequent deep-RL methods demonstrated that reward models trained on pairwise feedback routinely outperform hand-designed objectives in complex domains such as games and robotics~\citep{christiano2017deep}. The same feedback-driven paradigm transformed language-model alignment: approaches like preference-guided PPO~\citep{ziegler2020fine}, recursive reward modeling~\citep{wu2021recursively}, and instruction-tuned systems such as InstructGPT~\citep{ouyang2022training} materially improved helpfulness, safety, and factuality. These striking empirical gains have motivated a fast-growing theoretical literature that seeks to formalize when and why RLHF works, and to quantify its statistical and computational limits~\citep{xiong2024iterative,ji2025ai,kaufmann2024survey}.

Theoretical progress began with \citet{novoseller2020dueling}, who framed learning from trajectory comparisons as a regret-minimization problem solved via Dueling Posterior Sampling. Follow-up methods — PPS and PEPS \citep{xu2020preference}, PR-LSVI \citep{wu2024making}, and PARL \citep{chakraborty2024parl} — provided regret bounds and convergence guarantees. Other contributions extended the model of feedback and dynamics \citep{chen2022human}, derived minimax regret rates \citep{saha2023dueling}, and analyzed policy optimization under linear and neural approximations \citep{du2024exploration}. Work has also broadened feedback models \citep{zhu2023principled,zhan2024provable}, developed interactive and feedback-efficient protocols \citep{kong2022provably}, and given off-policy, privacy-preserving, and reduction-based results \citep{li2023reinforcement,chowdhury2023differentially,wang2023rlhf}. Most recently, researchers have sharpened safety and robustness guarantees, producing high-confidence constraint satisfaction and provable convergence under unknown preference mappings \citep{chittepu2025reinforcement,zhang2025provable}.

Many alignment tasks involve several preference sources or evaluation
criteria. One may maximize utility for a target population while requiring
a minimum expected reward for another stakeholder group, or optimize
safety subject to a helpfulness requirement. We study one such constraint
in the main text and discuss multiple constraints in
Section~\ref{sec:extension}. The two-oracle problem is
\begin{equation*}
\max_{\pi}\;\mathbb{E}_\pi[r_1^*]
-\eta\mathbb{E}_{x\sim d_0}
D_{\mathrm{KL}}(\pi(\cdot\mid x)\|\pi_0(\cdot\mid x))
\quad\text{s.t.}\quad \mathbb{E}_\pi[r_2^*]\ge J_{\min}.
\end{equation*}
where $r_1^*$ captures primary-user utility, $r_2^*$ measures protected-group welfare, $\pi_0$ is a reference policy, $J_{\min}$ is the required minimum for the protected group, and $\eta > 0$ trades off utility against deviation from $\pi_0$.

This is a one-stage constrained stochastic control problem, with prompts
as contexts and complete responses as actions. Its connection to control
is the synthesis of a regularized policy subject to expected performance
constraints. The guarantee concerns an expectation under $d_0$. It is
not a per-prompt, per-response, or distribution-shift guarantee. The
protected-group interpretation requires preferences from an identified
group and a justified threshold. Separate safety and helpfulness labels
instead instantiate two evaluation criteria.

Our focus is an explicit finite-iteration analysis of dual optimization
for offline constrained preference learning. SACPO~\citep{wachi2024stepwise}
also studies constrained KL-regularized alignment and gives statistical
optimality and violation bounds. We complement this analysis with
computable multiplier domains, direct certificates for the policy
returned by the dual solver, and explicit optimization rates.

In this paper, we investigate \emph{constrained RLHF} in the \emph{offline} setting, where the learner receives a fixed batch of preference data consisting of paired comparisons annotated by both the primary and the protected populations. The learner must synthesize a policy without any further environment interaction. Our contributions are threefold. First, we analyze how offline
preference-estimation uncertainty affects constrained KL-regularized
policy optimization. Second, we give guarantees for the learned
Gibbs policy, including a sharper scalar last-iterate analysis and a
conservative variant with certified feasibility. Third, we describe
multi-constraint extensions and identify the additional conditions
needed for general divergence regularization.

\section{Related Work}
We survey literature across three axes: theoretical RLHF, constrained RL, and the emerging intersection of alignment and safety, highlighting the gap addressed by our study.

\textbf{RLHF:}
Theoretical guarantees for RLHF have matured in online \citep{novoseller2020dueling, saha2023dueling} and offline regimes following empirical successes in robotics \citep{christiano2017deep} and LLMs \citep{ouyang2022training}. In the offline regime, \citet{zhu2023principled} provided the first finite-sample guarantees for policies trained via maximum likelihood estimation under the Bradley–Terry model with linear rewards~\citep{bradley1952rank}. Building on this, \citet{zhan2024provable} introduced the FREEHAND algorithm, which generalizes earlier approaches by allowing a broader class of reward functions and feedback models. \citet{li2023reinforcement} provided the first off-policy analysis for offline RLHF via the DCPPO algorithm. \citet{zhu2024iterative} addressed overfitting and overoptimization in reward learning and proposed the IDS algorithm for improved robustness. \citet{chowdhury2023differentially} studied privacy-preserving reward learning in preference-based RL. \citet{wang2023rlhf} presented reduction-based approaches that adapt reward-based RL algorithms to the RLHF setting, showing how theoretical guarantees transfer. \citet{zhou2025unified} unified the analysis of privacy and robustness in offline RLHF.

\textbf{Constraint RL:} Constrained RL historically traces back to CMDP formulations \citep{kolesar1970markovian,keith1989randomized,altman1999cmdp}. Foundational optimization ideas such as Lagrangian relaxation \citep{everett1963generalized,shapiro1979mathematical} and primal-dual updates \citep{altman1998constrained,efroni2020exploration,bertsekas2016nonlinear} motivated algorithms that embed multipliers directly into RL procedures \citep{zheng2020Constrained,ding2020natural,ying2022dual}. In the offline setting, rigorous guarantees are now available for primal–dual critic methods \citep{hong2024primal}, LP-based algorithms under partial coverage \citep{hong2025offline}, and multi-constraint primal policy optimization \citep{guan2024poce}.

\textbf{Alignment and Fairness:} Work in this area includes axiomatic treatments \citep{ge2024axioms}, max–min welfare \citep{chakraborty2024maxmin}, and group-robustness via uncertainty sets \citep{ramesh2024group}, while practical methods use multi-objective DPO \citep{zhou2024beyond} or Lagrangian PPO-style tuning \citep{dai2024safe}. These primarily emphasize empirical results or axioms. We complement them with a finite-sample analysis of constrained offline preference learning.

\section{Formulation}\label{sec:formulation}
In this section, we present the problem formulation. We extend the standard preference-based RLHF setup \citep{ouyang2022training,xiong2024iterative} to the constrained setting. Let $\mathcal{X}$ be the finite prompt space and $\mathcal{A}$ the finite response space.  A prompt $x\sim d_0$ is first sampled from a fixed distribution $d_0$.  Conditional on $x$, two independent responses $a,a'\in\mathcal{A}$ from the reference policy $\pi_0$ are produced.  Feedback is then collected from two human-preference oracles: a target-population oracle $o_1$ and a protected-group oracle $o_2$. Each oracle’s binary preference, following the Bradley–Terry model, is modeled as a Bernoulli random variable whose success probability is given by a logistic function of the latent reward gap:
\begin{equation*}
    o_k(x,a,a') \sim \operatorname{Ber}\left(\sigma(r_k^*(x,a)-r_k^*(x,a'))\right)
\end{equation*}  
 for $k\in\{1,2\}$ where $r_k^*:\mathcal{X}\times\mathcal{A}\to\mathbb{R}$ is unknown reward function for oracle $k$ and $\sigma(z)=(1+e^{-z})^{-1}$ is the logistic link.  Specifically, $o_k(x,a,a') = 1$ indicates that oracle $o_k$ prefers action $a$ over action $a'$, denoted as $a \succ_{k} a'$.

We operate in the offline setting, where the learner has access to a dataset
\begin{align*}
    \mathcal{D}_N = \{(x_i, a_i^{(1)}, a_i^{(2)}, y_{i,1}, y_{i,2})\}_{i=1}^N
\end{align*}
consisting of prompts, response pairs, and corresponding binary preferences. Each tuple $(x_i, a_i^{(1)}, a_i^{(2)})$ is drawn \emph{i.i.d.} from the distribution $(x, a, a') \sim \mu_0 \coloneqq d_0(x)\pi_0(a|x)\pi_0(a'|x)$. For each triplet, the preferences $y_{i,1}, y_{i,2} \in \{0,1\}$ are sampled from the two oracles $o_1$ and $o_2$. More specifically,
\begin{align*}
     \mathbb{P}(y_{i,k} = 1 | x_i,a^{(1)}_i, a^{(2)}_i) &=  \mathbb{P}\big(a^{(1)}_i \succ_{k} a^{(2)}_i \big| x_i,a^{(1)}_i, a^{(2)}_i\big) \\ &= \sigma(r^*_k(x_i, a^{(1)}_i) - r^*_k(x_i, a^{(2)}_i))
\end{align*}

The learner seeks a policy $\pi^*: \mathcal{X} \to \Delta(\mathcal{A})$ that maximizes the expected reward of the target population, remains sufficiently close to a reference policy $\pi_0$, and ensures a minimum level of welfare for the protected group. Here, $\Delta(\mathcal{A})$ denotes the set of all probability distributions over the response set $\mathcal{A}$. Formally, the goal is to solve the following constrained optimization problem:
\begin{align*}
    \underset{\pi\in\Pi}{\max} ~ & \mathbb{E}_{x \sim d_0}\!\!\left[\mathbb{E}_{a \sim \pi(\cdot|x)}\left[r_1^*(x,a)\right]\! 
    - \eta D_{\mathrm{KL}}\left(\pi(\cdot|x)\|\pi_0(\cdot|x)\right) \right] \\
    \text{s.t.} ~ & \mathbb{E}_{x \sim d_0}\!\!\left[\mathbb{E}_{a \sim \pi(\cdot|x)}\left[r_2^*(x,a)\right]\right] \ge J_{\min},
\end{align*}
where $\Pi$ denotes the set of all policies $\pi:\mathcal{X} \to \Delta(\mathcal{A})$, $J_{\min}$ denotes the minimum acceptable reward for the protected group, and $\eta > 0$ controls the trade-off between utility and divergence from the reference policy.

We restrict $\mathcal{X}$ to the support of $d_0$, so that $d_0(x)>0$ for every $x\in\mathcal{X}$. Our optimization guarantees assume exact evaluation of expectations under $d_0$ and the Gibbs policies. Thus the analysis applies to a known finite prompt distribution, or to a fixed empirical prompt distribution when that distribution defines the objective. Generalization from an empirical prompt distribution to a population distribution requires an additional analysis.

Here, we restrict our formulation to one protected group and defer the generalization to Section \ref{sec:extension}. For notational convenience, we rewrite the problem in abstract form:
\begin{align}\label{eq:primal}
    \underset{\pi \in \Pi}{\max}~~ J(\pi) \quad\text{s.t.}\quad c(\pi) \leq 0,
\end{align}
where $J(\pi)$ denotes the regularized target reward objective, and
\begin{align*}
    c(\pi) \coloneqq J_{\min} - \mathbb{E}_{x \sim d_0}\left[\mathbb{E}_{a \sim \pi(\cdot | x)}[r^*_2(x,a)] \right]
\end{align*}
is the constraint function. Throughout the paper we adopt the shorthand notation
\begin{align*}
    \mathbb{E}_\pi \coloneqq \mathbb{E}_{x \sim d_0}\mathbb{E}_{a \sim \pi(\cdot | x)} \quad , \quad \mathbb{E} \coloneqq \mathbb{E}_{x \sim d_0}
\end{align*}
unless stated otherwise.

Following~\citet{xiong2024iterative,zhu2023principled}, we impose the following assumptions on the reference policy and the reward functions.

\begin{assumption}[Full coverage]\label{assump:fullSupport}
For every $x \in \mathcal{X}$, the reference policy $\pi_0(\cdot \mid x)$ has full support over the finite action space $\mathcal{A}$.
\end{assumption}

\begin{assumption}[Linear reward]\label{assump:linearReward}
For each $k\in\{1,2\}$, the latent reward is assumed to be linear in a known feature map $\phi$, i.e., $r_k^*(x,a)=\langle\theta_k^*,\phi(x,a)\rangle,$
with $\|\phi(x,a)\|\le 1$ for all $(x,a)$ and $\|\theta_k^*\|\le B$.
\end{assumption}

\begin{assumption}[Identifiability]\label{assump:identifiability}
Let $\Delta(x;a,a')=\phi(x,a)-\phi(x,a')$ and
\[
\Sigma_\infty
=\mathbb{E}_{x\sim d_0,\,a,a'\sim\pi_0(\cdot\mid x)}
 [\Delta(x;a,a')\Delta(x;a,a')^\top].
\]
We assume $\Sigma_\infty\succ0$. Under full reference support, this
is equivalent to the feature differences at prompts in the support
of $d_0$ spanning $\mathbb{R}^d$.
\end{assumption}
Pairwise preferences identify reward differences. If $v$ belongs to
the nullspace of $\Sigma_\infty$, then adding $tv$ to a parameter
does not change its comparison probabilities, wherever the modified
parameter remains in the parameter set. Thus, without identifiability,
the likelihood can have multiple equivalent parameters. Assumption~\ref{assump:identifiability} excludes these directions within our
specified linear model, however does not guarantee that a finite-sample
design matrix has full rank.

An absolute welfare threshold also requires a specified reward origin.
Replacing $r_2(x,a)$ by $r_2(x,a)+b(x)$ leaves every pairwise preference
unchanged but changes the absolute welfare level. Accordingly, $J_{\min}$
must be specified in the same calibrated reward representation as
$r_2^*$. A baseline-relative requirement,
$\mathbb{E}_\pi r_2^* - \mathbb{E}_{\pi_0}r_2^*\ge\Delta_{\min}$,
is invariant to these prompt-dependent offsets, but defines a different
constraint from the absolute threshold studied here.

\paragraph{Scope of assumptions.}
Linear realizability requires the true oracle rewards to belong to
the span of a fixed, known feature map. The map itself can be
nonlinear, for example a frozen neural encoder, but a fitted linear
head does not by itself establish realizability. The MLE confidence
bound does not automatically apply to jointly trained or
overparameterized neural reward models. The Gibbs reduction and the
learned-policy certificate can be used with other bounded reward
models if suitable uniform error bounds are supplied. Under
misspecification, such bounds must also account for approximation
error; it need not vanish with more preference data.

Full reference support does not imply good statistical coverage:
small comparison-covariance eigenvalues can still make the guarantees
weak. If reference probabilities vanish, the finite-KL optimization
can be restricted to the reference support, without guarantees for
excluded actions. Exact Gibbs evaluation and prompt expectations can
be tractable for finite candidate sets, but are substantial oracle
assumptions for unrestricted language generation. Approximate policy
improvement and sampled gradients require separate error control.

The i.i.d.\ sampling restriction $(x,a,a')\sim\mu_0$ can be replaced
by any fixed comparison design distribution $\mu$, with
$\Sigma_\infty$ defined under $\mu$. The conditional MLE, matrix
concentration, and deterministic policy arguments continue to hold
under adequate feature coverage; trace-based quantities such as the
identity for $\operatorname{tr}\Sigma_\infty$ must then be derived
from the reference covariance separately. This accommodates more
general offline comparison datasets, but does not address
reward-model misspecification or unknown prompt expectations.

\section{Analysis}
In this section, we analyze the constrained optimization problem introduced above. First, using the offline preference dataset $\mathcal{D}_N$, we obtain maximum-likelihood estimators $\widehat{r}_1$ and $\widehat{r}_2$ of the latent reward functions $r_1^*$ and $r_2^*$. We then construct the corresponding Lagrangian, derive the associated dual problem, and verify the convexity conditions required for strong duality.

\subsection{Reward Estimation}
Under Assumptions~\ref{assump:linearReward} and \ref{assump:identifiability}, given $\{(x_i,a^{(1)}_i, a^{(2)}_i)\}_{i=1}^N$, the log-likelihood of the preferences collected from oracle $k\in\{1,2\}$ is given by
\begin{align*}
    \ell_{\mathcal{D}_N}^{(k)}(\theta_k) = \sum_{i=1}^N \ell_{i}^{(k)}(\theta_k),
\end{align*}
where the individual sample log-likelihood is
$\ell_{i}^{(k)}(\theta_k) = y_{i,k}\log \sigma(\langle \theta_k, \Delta_i \rangle) + (1 - y_{i,k}) \log \sigma(-\langle \theta_k, \Delta_i \rangle),
$ and $ \Delta_i \coloneqq \phi(x_i, a_i^{(1)}) - \phi(x_i, a_i^{(2)}) $ represents the feature difference vector. For $N\ge1$, we use norm-constrained maximum-likelihood estimators
\[
\widehat\theta_k\in
\arg\max_{\|\theta\|_2\le B}\ell_{\mathcal{D}_N}^{(k)}(\theta),
\qquad k\in\{1,2\}.
\]
In particular, both the true and estimated rewards are bounded in
absolute value by $B$. Define
\[
\Sigma_N=\frac1N\sum_{i=1}^N\Delta_i\Delta_i^\top,
\qquad V_N=\Sigma_N+\nu I,
\]
where $\nu\ge0$ is chosen so that $V_N\succ0$. This auxiliary
regularization defines the confidence norm. Since $\|\Delta_i\|_2\le2$,
the logistic curvature on the parameter set is bounded below by
\[
\gamma_B=\min_{|z|\le2B}\sigma(z)(1-\sigma(z))
=\frac{1}{2+e^{2B}+e^{-2B}}.
\]
The constrained-MLE confidence bound (Lemma~3.1 of \citet{zhu2023principled}) gives, simultaneously for
$k=1,2$, with probability at least $1-2\delta/3$,
\[
\|\widehat\theta_k-\theta_k^*\|_{V_N}\le\beta_N,
\qquad
\beta_N=C\sqrt{
\frac{d+\log(3/\delta)}{\gamma_B^2N}+\nu B^2},
\]
where $C$ is a universal numerical constant. A refinement of this radius, obtained via a Fisher--Bernstein
argument, is given in Appendix~\ref{app:fisher-bernstein}. The remaining
$\delta/3$ probability budget is reserved for the population
covariance bound below. No independence between the two oracle
labels for the same comparison is needed, as a union bound combines
the two marginal confidence events. We assume that the complete labeled comparison tuples are independent
across sample indices, and within-tuple oracle dependence is unrestricted.

All subsequent statistical statements are made on this common
confidence event. For compact bounds, write
\[
\varepsilon_k
=\min\left\{2B,\,
\beta_N\max_{x,a}\|\phi(x,a)\|_{V_N^{-1}}\right\},
\quad k=1,2.
\]
Then $|\widehat r_k(x,a)-r_k^*(x,a)|\le\varepsilon_k$
simultaneously for all $(x,a)$. In particular, one may use the
weaker but simpler bound
$\varepsilon_k\le\min\{2B,\beta_N/
\sqrt{\lambda_{\min}(V_N)}\}$.
For a policy $\pi$, a sharper expectation certificate is
\[
b_k(\pi)=\beta_N\left\|
\mathbb{E}_\pi\phi(x,a)\right\|_{V_N^{-1}},
\qquad
|\mathbb{E}_\pi(\widehat r_k-r_k^*)| \le b_k(\pi).
\]
These bounds hold for data-dependent policies as well, because
they are consequences of a single parameter-confidence event.

The confidence argument permits arbitrary dependence between oracle
labels for the same comparison, provided their marginal Bradley--Terry
models are correct. Exploiting such dependence would require an
explicit joint model or additional structure, and is not needed for
our guarantees.

\subsection{Dual Problem Analysis}
Let $\mathcal{L}(\pi,\lambda)$ be the Lagrangian of the constrained problem \eqref{eq:primal}. We have
\begin{align*}
    \mathcal{L}(\pi,\lambda) &= \mathbb{E}_\pi[r^*_1(x,a) + \lambda r^*_2(x,a)] \\ 
     &- \eta\mathbb{E}[D_{\mathrm{KL}}(\pi(\cdot | x)\|\pi_0(\cdot | x))] - \lambda J_{\min}
\end{align*}
The dual problem is therefore $\min_{\lambda\geq 0} \max_\pi  \mathcal{L}(\pi,\lambda).$

Observe that the set of all policies $\Pi$ is a convex set. Moreover, for every $x \in \mathcal{X}$, the \textrm{KL}-divergence $D_{\mathrm{KL}}(\pi(\cdot\mid x)\|\pi_0(\cdot\mid x))$ is strictly convex in $\pi(\cdot\mid x)$. Hence, the primal problem is a strictly concave maximization with an affine constraint, and under Slater’s condition, strong duality holds.

\begin{assumption}[Slater’s condition]\label{assump:slater}
There exists a policy $\tilde\pi\in\Pi$ and a slack $\rho>0$ such that
\begin{align*}
    \mathbb{E}_{\tilde\pi}\big[r_2^{*}(x,a)\big]\geq J_{\min}+\rho.    
\end{align*}
\end{assumption}

Choose a deterministic tie-breaking rule and let
$a_{\rm gr}(x)\in\arg\max_a\widehat r_2(x,a)$,
$\widetilde\pi(a\mid x)=\mathbf{1}\{a=a_{\rm gr}(x)\}$.

\begin{corollary}[Slater certificate]
\label{cor:slaterData}
On the confidence event we have $\mathbb{E}_{\widetilde\pi}r_2^*
\ge \mathbb{E}_{\widetilde\pi}\widehat r_2
-b_2(\widetilde\pi).$ Consequently, if $\rho_{\rm cert}:=
\mathbb{E}_{\widetilde\pi}\widehat r_2
-J_{\min}-b_2(\widetilde\pi)>0,$ then Slater's condition holds with slack $\rho_{\rm cert}$.
\end{corollary}

This estimate can be used to bound the optimal dual parameter. Since $J(\pi)$ is strictly concave and the feasible set is convex, the constrained problem admits a unique optimal policy $\pi^\star\in\Pi$.

Consider the dual function $g(\lambda) = {\max_\pi} \mathcal{L}(\pi,\lambda)$. By standard results for KL-regularized objectives (e.g., \citet{zhang2023mathematical}), the maximizer $\pi^*_{\lambda} = {\arg \max}_\pi \mathcal{L}(\pi,\lambda)$ has the Gibbs (Boltzmann) form
\begin{align*}
    \pi^*_\lambda(a|x) = \frac{\pi_0(a|x)\exp\left(\frac{1}{\eta}\left<\theta^*_1 + \lambda \theta^*_2,\phi(x,a) \right> \right)}{Z_\lambda(x)}
\end{align*}
where $Z_\lambda(x)$ is the normalizing constant, also referred to as the partition function. The corresponding dual function is
\[
g(\lambda)=\eta\mathbb{E}_x\log Z_\lambda(x)-\lambda J_{\min}.
\]
The estimated policy $\widehat\pi_\lambda$ and dual
$\widehat g(\lambda)$ are defined by replacing the true rewards
with their estimates. Hats always identify estimated-reward objects.
Having a closed-form policy optimizer for each fixed dual variable
reduces the remaining optimization to the scalar dual problem. Following \citet{xiong2024iterative}, we assume access to a ``Policy Improvement Oracle''.
\begin{definition}[Policy Improvement Oracle~\citep{xiong2024iterative}]
    For reward function $r : \mathcal{X} \times \mathcal{A} \to \mathbb{R}$ and a reference policy $\pi_0$, for all $x \in \mathcal{X}$, we can compute the Gibbs policy:
    \begin{align*}
        \pi_r(\cdot \mid x) & \coloneqq \arg\max_{\pi \in \Pi} \mathbb{E}_{a \sim \pi(\cdot \mid x)} \left[ r(x, a) - \eta \log \frac{\pi(a \mid x)}{\pi_0(a \mid x)} \right] \\ &\propto \pi_0(\cdot \mid x) \cdot \exp\left( \frac{1}{\eta} r(x, \cdot) \right).       
    \end{align*}
\end{definition}

Next, we analyze the properties of the dual function $g(\cdot)$. By the envelope theorem, we have
\begin{align*}
    g'(\lambda) = \mathbb{E}_{\pi^*_\lambda}[r^*_2(x,a)] - J_{\min}.
\end{align*}
Since $\pi^*_\lambda$ belongs to an exponential family, its mean parameter is Lipschitz continuous given boundedness of its sufficient statistics. Consequently, the derivative of the dual function $g'(\lambda)$ is Lipschitz continuous~\citep{wainwright_graphical_2007,brown_fundamentals_1986}.
\begin{lemma}\label{lem:gradLips}
For every finite $\lambda\ge0$,
\[
g''(\lambda)=\frac1\eta\mathbb{E}_x
\operatorname{Var}_{\pi_\lambda^*(\cdot\mid x)}(r_2^*(x,a)).
\]
Hence $g'$ is Lipschitz with constant
\[
L_{\rm range}
=\frac1{4\eta}\mathbb{E}_x
\left(\max_a r_2^*(x,a)-\min_a r_2^*(x,a)\right)^2
\le\frac{B^2}{\eta}.
\]
The estimated dual has the computable Lipschitz bound
\[
\widehat L_{\rm range}
=\frac1{4\eta}\mathbb{E}_x
\left(\max_a\widehat r_2(x,a)-\min_a\widehat r_2(x,a)\right)^2.
\]
\end{lemma}

The same properties also hold for the empirical dual function $\widehat{g}(\lambda) = \max_\pi \mathcal{L}(\pi, \lambda; \widehat{\theta}_1, \widehat{\theta}_2)$, where $\mathcal{L}(\pi, \lambda; \widehat{\theta}_1, \widehat{\theta}_2)$ denotes the Lagrangian of the primal problem~\eqref{eq:primal} with the true parameters $\theta^*_1$ and $\theta^*_2$ replaced by their statistical estimates. In this case, $\pi^*_\lambda$ is replaced in the proof by $\widehat{\pi}_\lambda$, the policy that attains the maximum in $\widehat{g}(\lambda)$. Next, we quantify the gap between $g(\lambda)$ and $\widehat{g}(\lambda)$, as well as their derivatives, in terms of the estimation errors of $\theta^*_1$, $\theta^*_2$. These bounds make explicit how statistical uncertainty propagates through the dual program.

\begin{lemma}
\label{dual concentrations}
On the common confidence event, simultaneously for every $\lambda\ge0$,
\[
|\widehat{g}(\lambda)-g(\lambda)|
\le\varepsilon_1+\lambda\varepsilon_2.
\]
Writing $e_k=\|\widehat\theta_k-\theta_k^*\|_2$, we also have
\[
|\widehat{g}'(\lambda)-g'(\lambda)|
\le\varepsilon_2+
\min\left\{2B,\frac{B}{\eta}(e_1+\lambda e_2)\right\},
\qquad
e_k\le\frac{\beta_N}{\sqrt{\lambda_{\min}(V_N)}}.
\]
In contrast, evaluating the same learned policy under the two reward
models only incurs
\[
\left|\mathbb{E}_{\widehat\pi_\lambda}
(\widehat r_2-r_2^*)\right|
\le \bar b_2(\widehat\pi_\lambda)\le\varepsilon_2,
\qquad \bar b_k(\pi):=\min\{b_k(\pi),\varepsilon_k\}.
\]
\end{lemma}

\begin{lemma}
\label{lem:norm-change}
Let $s=\lambda_{\min}(\Sigma_\infty)>0$. For $0<a<1$,
\[
\mathbb{P}\left\{
\lambda_{\min}(\Sigma_N)<(1-a)s
\right\}
\le d\exp\left(-\frac{Nsa^2}{8}\right).
\]
Consequently, if
\[
N\ge\frac{32}{s}\log\frac{3d}{\delta},
\]
then with probability at least $1-\delta/3$,
\[
\lambda_{\min}(V_N)\ge s/2+\nu.
\]
On the intersection with the two MLE confidence events, of probability
at least $1-\delta$, all data-dependent bounds therefore admit the
population substitution
\[
\varepsilon_k\le
\min\left\{2B,\frac{\beta_N}{\sqrt{s/2+\nu}}\right\}.
\]
\end{lemma}

An $N^{-1/2}$ statistical rate requires the auxiliary regularization
term to vanish at a comparable rate, for example $\nu=O(1/N)$,
or $\nu=0$ when $\Sigma_N\succ0$. Keeping $\nu$ fixed leaves
a non-vanishing term in this particular confidence bound.

For tighter instance-specific coverage, define
$\kappa=\operatorname*{ess\,sup}\Delta^\top\Sigma_\infty^{-1}\Delta\le4/s$.
Whitening and matrix Chernoff give
\[
\Sigma_N\succeq(1-a)\Sigma_\infty,
\qquad
a=\sqrt{2\kappa\log(3d/\delta)/N}<1
\]
with probability at least $1-\delta/3$. This can replace the simpler
$s/2$ bound when the leverage parameter $\kappa$ is appreciably
smaller than $4/s$.

\begin{proposition}
\label{prop:strict-convexity}
The KL dual is convex. If
\[
\mathbb{P}_{x\sim d_0}
\{r_2^*(x,\cdot)\text{ is nonconstant}\}>0,
\]
then it is strongly convex on every compact interval $[0,R]$, with
\[
m_g(R)=\frac1\eta\inf_{\lambda\in[0,R]}
\mathbb{E}_x\operatorname{Var}_{\pi_\lambda^*(\cdot\mid x)}(r_2^*)>0.
\]
Under Assumption~\ref{assump:identifiability}, the nonconstancy
condition is equivalent to $\theta_2^*\ne0$, and
\[
m_g(R)\ge
\frac{e^{-2B(1+R)/\eta}}{2\eta}
(\theta_2^*)^\top\Sigma_\infty\theta_2^*.
\]
\end{proposition}
The corresponding empirical statement requires nonconstancy of
$\widehat r_2$, not merely nonconstancy of $r_2^*$.
The main optimization guarantees below do not require strong convexity.
With these properties in place, we can now present the main result of this section.

\begin{theorem}\label{thm:lambdaBound}
    Under Slater's condition, a dual minimizer $\lambda^*\ge0$ exists,
    and the unique primal optimizer is $\pi^*=\pi^*_{\lambda^*}$.
    For a Slater policy $\widetilde\pi$ with slack $\rho>0$,
    \[
    \lambda^*\le
    \frac{g(0)-J(\widetilde\pi)}{\rho}
    \le\frac{B-J(\widetilde\pi)}{\rho}.
    \]
    On the reward-confidence event, if
    $\widehat w_2(\widetilde\pi)-J_{\min}-\varepsilon_2>0$,
    where $\widehat w_2(\pi)=\mathbb{E}_\pi\widehat r_2$, then
    \[
    \lambda^*\le R_{\rm true}:=
    \frac{\widehat g(0)-\widehat J(\widetilde\pi)+2\varepsilon_1}
    {\widehat w_2(\widetilde\pi)-J_{\min}-\varepsilon_2}.
    \]
    Here $\widehat J$ is $J$ with $r_1^*$ replaced by $\widehat r_1$.
    
    For an empirical threshold increment $\tau\ge0$, define
    \[
    \widehat g_\tau(\lambda)=\widehat g(\lambda)-\lambda\tau,
    \qquad
    \widehat\rho_\tau=
    \widehat w_2(\widetilde\pi)-J_{\min}-\tau.
    \]
    If $\widehat\rho_\tau>0$, every minimizer of
    $\widehat g_\tau$ over $\lambda\ge0$ satisfies
    \[
    \widehat\lambda_\tau^*\le R_\tau:=
    \frac{\widehat g(0)-\widehat J(\widetilde\pi)}
    {\widehat\rho_\tau}.
    \]
    In particular, $R_0$ is an empirical radius for the unmodified
    constraint, whereas $R_{\varepsilon_2}$ is an empirical radius
    for the conservative constraint.
\end{theorem}

If $m_g(R)>0$ and $\lambda^*\le R$, then
\[
\lambda^*\le\frac{[-g'(0)]_+}{m_g(R)}.
\]
If $g'(0)\ge0$, then $\lambda^*=0$, otherwise the
minimizer is positive and $g'(\lambda^*)=0$. Strong convexity
is used only on $[0,R]$. A data-dependent numerator follows from
the derivative-perturbation bound in Lemma~\ref{dual concentrations}.

\section{Algorithm}
After reward estimation, we minimize the empirical scalar dual.
The projection radius must contain an empirical dual minimizer;
a radius for the true multiplier alone is insufficient. We use
$R\ge R_0$ from Theorem~\ref{thm:lambdaBound}, initialize at zero,
and choose a positive smoothness bound
$\widehat L\ge\widehat L_{\rm range}$.
The last iterate is returned. Monotonicity of the scalar dual derivative
allows sharper policy guarantees than a generic averaged-iterate argument.

\begin{algorithm}[t]
\caption{Scalar dual projected gradient descent}
\label{alg:pgd-dual}
\begin{algorithmic}[1]
\REQUIRE Estimated rewards $\widehat r_1,\widehat r_2$;
threshold $J_{\min}$; empirical radius $R\ge R_0$;
positive smoothness bound $\widehat L$; iteration count $T$
\STATE $\lambda_0\gets0$
\IF{$\widehat g'(0)\ge0$}
  \STATE \textbf{Return} $0$ and $\widehat\pi_0$
\ENDIF
\FOR{$t=0$ to $T-1$}
  \STATE Form $\widehat\pi_{\lambda_t}(a\mid x)
  \propto\pi_0(a\mid x)
  \exp((\widehat r_1(x,a)+\lambda_t\widehat r_2(x,a))/\eta)$
  \STATE $h_t\gets\mathbb{E}_{\widehat\pi_{\lambda_t}}
  \widehat r_2-J_{\min}$
  \STATE $\lambda_{t+1}\gets
  \operatorname{Proj}_{[0,R]}(\lambda_t-h_t/\widehat L)$
\ENDFOR
\STATE \textbf{Return} $\lambda_T$ and $\widehat\pi_{\lambda_T}$
\end{algorithmic}
\end{algorithm}

\begin{lemma}
\label{lem:policy-certificate}
    On the common reward-confidence event, define
    \[
    A_N=\min\left\{2\varepsilon_1,
    \frac{\varepsilon_1^2}{2\eta}\right\}.
    \]
    For any $\lambda\ge0$ and any $\tau\ge0$, let $h_\tau(\lambda)
    =\widehat w_2(\widehat\pi_\lambda)-J_{\min}-\tau$. Then, simultaneously for all these choices,
    \begin{align*}
    [c(\widehat\pi_\lambda)]_+
    &\le[\varepsilon_2-\tau-h_\tau(\lambda)]_+,\\
    J(\pi^*)-J(\widehat\pi_\lambda)
    &\le A_N+\lambda\big(\varepsilon_2+\tau+h_\tau(\lambda)\big).
    \end{align*}
\end{lemma}

\begin{theorem}
\label{thm:main}
    On the common reward-confidence event, suppose the empirical
    constraint is strictly feasible and $R\ge R_0$. Run
    Algorithm~\ref{alg:pgd-dual} with exact expectations and
    $\widehat L\ge\widehat L_{\rm range}>0$.
    For every integer $T\ge0$,
    \begin{align*}
    \widehat g(\lambda_T)-\min_{\lambda\ge0}\widehat g(\lambda)
    &\le\frac{\widehat L R^2}{4T+2},\\
    [c(\widehat\pi_{\lambda_T})]_+
    &\le\varepsilon_2+\frac{\widehat L R}{T+1},\\
    J(\pi^*)-J(\widehat\pi_{\lambda_T})
    &\le A_N+R\varepsilon_2.
    \end{align*}
    The objective comparison is one-sided and must be interpreted together
    with the feasibility bound: an infeasible policy may have larger
    objective than the feasible optimum.
    If $\widehat g'(0)\ge0$, the algorithm returns $\widehat\pi_0$,
    its empirical dual error is zero, and its objective gap is at most $A_N$.
    The bounds hold with probability at least $1-\delta$.
\end{theorem}
If also $R\ge R_{\rm true}$ (so that the projection radius contains
both the true and the empirical minimizers), then
\[
g(\lambda_T)-g(\lambda^*)
\le 2\varepsilon_1+2R\varepsilon_2
+\frac{\widehat L R^2}{4T+2}.
\]
The policy bounds separate reward estimation uncertainty
from finite-iteration optimization. The constraint violation has
an $O(1/T)$ optimization term.
For fixed coverage and reward bounds, the target-reward contribution
to the one-sided objective gap is quadratic in its estimation error,
while constraint-reward uncertainty contributes linearly through
the multiplier radius. A conservative variant below provides exact
feasibility on the confidence event by bisecting the boundary of
the tightened constraint, at the cost of a conservative radius.

\begin{algorithm}[t]
\caption{Conservative feasible-endpoint bisection}
\label{alg:certified-bisection}
\begin{algorithmic}[1]
\REQUIRE Estimated rewards $\widehat r_1,\widehat r_2$;
confidence radius $\varepsilon_2$; threshold $J_{\min}$;
iteration count $T$
\STATE $\tau\gets\varepsilon_2$;
$h(\lambda)\gets\widehat w_2(\widehat\pi_\lambda)
-J_{\min}-\tau$
\IF{$h(0)\ge0$}
  \STATE \textbf{Return} $0$ and $\widehat\pi_0$
\ENDIF
\STATE Find a policy $\widetilde\pi$ with
$\widehat\rho_\tau=\widehat w_2(\widetilde\pi)
-J_{\min}-\tau>0$
\STATE If no such policy is certified, terminate and report that the
conservative problem has not been certified strictly feasible
\STATE $R\gets(\widehat g(0)-\widehat J(\widetilde\pi))
/\widehat\rho_\tau$; $\ell_0\gets0$; $u_0\gets R$
\FOR{$t=0$ to $T-1$}
  \STATE $m_t\gets(\ell_t+u_t)/2$
  \IF{$h(m_t)<0$}
    \STATE $(\ell_{t+1},u_{t+1})\gets(m_t,u_t)$
  \ELSE
    \STATE $(\ell_{t+1},u_{t+1})\gets(\ell_t,m_t)$
  \ENDIF
\ENDFOR
\STATE \textbf{Return} $u_T$ and $\widehat\pi_{u_T}$
\end{algorithmic}
\end{algorithm}

\begin{theorem}
\label{thm:certified}
On the common reward-confidence event, suppose either $h(0)\ge0$
or the algorithm finds a policy $\widetilde\pi$ with
$\widehat\rho_{\varepsilon_2}>0$ as required by
Algorithm~\ref{alg:certified-bisection}. The returned policy satisfies
$\mathbb{E}_{\widehat\pi_{u_T}}r_2^*\ge J_{\min}$.
In the nontrivial bisection branch, for any Lipschitz constant
$\widehat L$ of $h$,
\begin{align*}
0\le J(\pi^*)-J(\widehat\pi_{u_T})
&\le A_N+2u_T\varepsilon_2+u_T h(u_T)\\
&\le A_N+2R\varepsilon_2+\widehat L R^2 2^{-T}.
\end{align*}
For the initial feasible branch $u_T=0$, the objective bound is $A_N$.
\end{theorem}

A sufficient condition for the conservative empirical problem to
have positive slack is a true Slater margin $\rho>2\varepsilon_2$, though
failure of this condition does not imply infeasibility
of the original problem. In this single-constraint finite-action
setting, greedy maximization of $\widehat r_2$ decides whether a
strictly feasible $\widetilde\pi$ exists, and the algorithm
terminates rather than returning an uncertified policy. The
intermediate certificate $A_N+2u_T\varepsilon_2+u_Th(u_T)$ is
sharper than its worst-case iteration bound and supplies a stopping
rule: an optimization tolerance $\xi>0$ is met once
$\widehat LR^22^{-T}\le\xi$. This geometric root-search rate
requires neither strong convexity nor a lower curvature bound, and
it removes the need for a tuned step size in the scalar problem.
Exact expectation evaluation is essential, since the bisection
relies on the sign of $h$.

The implied rates are as follows. Take
$s=\lambda_{\min}(\Sigma_\infty)$, $\nu=O(1/N)$, and treat $B$ as
a constant. Up to numerical constants and the stated regularization
contribution, the eigenvalue certificate
(Lemma~\ref{lem:norm-change}) gives
\[
\varepsilon_k=O\!\left(
\frac1{\gamma_B}\sqrt{\frac{d+\log(1/\delta)}{Ns}}\right),
\]
and the guaranteed-feasibility bound of
Theorem~\ref{thm:certified} is
\[
O\!\left(
\frac{d+\log(1/\delta)}{\eta\gamma_B^2Ns}
+\frac{R}{\gamma_B}\sqrt{\frac{d+\log(1/\delta)}{Ns}}
+\widehat L R^2 2^{-T}\right).
\]
Only the target-reward estimation term improves to $O(1/N)$, through
the quadratic certificate $A_N=\min\{2\varepsilon_1,
\varepsilon_1^2/(2\eta)\}$. Constraint-reward uncertainty enters
linearly through $R\varepsilon_2$ and generally leaves an
$N^{-1/2}$ contribution in the full constrained objective gap.
Moreover, under unit-norm features and independent reference
responses,
$\operatorname{tr}\Sigma_\infty
=2\,\mathbb{E}_x\operatorname{tr}\operatorname{Cov}_{\pi_0}(\phi)
\le2$, so $s\le2/d$ and a constant that absorbs $s$ can hide
dimension dependence. The policy-mean certificates
$\bar b_k(\pi)$ and span-based constants such as
$\widehat L_{\rm range}$ and $A_N^{\rm span}$ are preferable when
they capture substantially better problem geometry than the
worst-case eigenvalue.

\section{Simulation}
We employ the PKU-SafeRLHF dataset~\citep{ji2024beavertails,ji2025pku}, which comprises approximately $74{,}000$ training and $7{,}000$ test examples. The data is structured as tuples of (prompt, response$_0$, response$_1$, safer, better), where prompts originate from diverse sources and responses are generated by Alpaca-7B, Alpaca2-7B, or Alpaca3-8B. The dataset records, for each comparison, which response is safer and which is more helpful; we encode binary preference labels with value $1$ indicating that response$_0$ is preferred, matching the likelihood convention used for estimation. We process the text by concatenating the prompt and response, encoding the sequence with Sentence-BERT (all-mpnet-base-v2)~\citep{reimers2019sentence}, and L2-normalizing each feature vector to satisfy $\|\phi(x,a)\|_2 = 1$, which enforces the feature-norm condition in Assumption~\ref{assump:linearReward}. We estimate the reward parameters $\widehat{\theta}_1, \widehat{\theta}_2 \in \mathbb{R}^{768}$ via regularized Bradley-Terry maximum likelihood, optimized using L-BFGS-B with $\lambda_{\mathrm{reg}} = 0.01$.

To construct the reference policy $\pi_0$, we generate forward passes using Alpaca-7B~\citep{taori2023alpaca} for each response pair. For every prompt and response, we compute the log-likelihood $\log \mathbb{P}(\text{response} \mid \text{prompt}) = \sum_{t} \log \mathbb{P}(\text{token}_t \mid \text{tokens}_{<t}, \text{prompt})$, which is normalized by the response length to account for token variation. Using the normalized log-likelihoods $\ell_0$ and $\ell_1$, we derive the policy probabilities via log-sum-exp normalization for numerical stability:
\begin{align*}
\ell_{\max} &= \max\{\ell_0, \ell_1\}, \\
\log Z &= \ell_{\max} + \log\big(\exp(\ell_0 - \ell_{\max}) + \exp(\ell_1 - \ell_{\max})\big),
\end{align*}
yielding $\pi_0(\text{response}_i \mid \text{prompt}) = \exp(\ell_i - \log Z)$ for $i \in \{0,1\}$. This defines a two-candidate reference distribution from length-normalized
Alpaca-7B scores. It is not the original language model's conditional
response distribution restricted to the pair. The provided pairs were
also not sampled independently from this constructed reference policy,
so this experiment does not instantiate the reference-sampling assumption
of the theoretical analysis; the fixed comparison-design extension of
Section~\ref{sec:formulation} is the appropriate framework. We then calibrate the constraint threshold as $J_{\min} = \mathbb{E}_{\pi_0}[\hat{r}_2] + 0.7 \cdot (\mathbb{E}_{\text{greedy}}[\hat{r}_2] - \mathbb{E}_{\pi_0}[\hat{r}_2])$, requiring the learned policy to bridge $70\%$ of the gap in helpfulness reward. We compute this model-based threshold using the training split
and keep it fixed for test evaluation. It specifies a fraction of the
estimated helpfulness improvement available within the candidate sets;
it is not an externally calibrated welfare floor. Finally, we implement Algorithm~\ref{alg:pgd-dual} with $\eta = 0.3$, initialization $\lambda_0 = 0$, projection radius $R = 100$, and a fixed horizon of $T = 1000$ iterations. To address the high-dimensional feature space and small normalized rewards, we employ an adaptive step size $\alpha_t$ that scales with the distance to the constraint boundary:
\begin{align*}
\alpha_t = \min\left\{\frac{\eta \cdot m(|J_{\min} - \mathbb{E}_{\widehat\pi_{\lambda_t}}[\widehat{r}_2]|)}{B^2\sqrt{\epsilon + \sum_{i=1}^{t} [\widehat g'(\lambda_i)]^2}}, 1.0\right\},
\end{align*}
where $B=\max\{\|\widehat{\theta}_1\|_2,\|\widehat{\theta}_2\|_2\}$
is the largest fitted parameter norm ($\approx 116$ in this experiment)
and $m(\cdot)$ is a piecewise-constant increasing function of the
distance $|J_{\min}-\mathbb{E}_{\widehat\pi_{\lambda_t}}[\widehat{r}_2]|$
to the constraint boundary, equal to $10{,}000$ for gaps above $0.1$,
$5{,}000$ for gaps in $(0.05,0.1]$, $1{,}000$ for gaps in $(0.01,0.05]$,
$500$ for gaps in $(0.005,0.01]$, and $100$ otherwise. We set $\epsilon = 10^{-8}$ as a regularization constant and cap the maximum step size at $\alpha_{\max} = 1.0$. This step-size rule is an implementation heuristic and is not covered by
the fixed-step theorem. The theoretical bound can be conservative because
\[
\widehat g''(\lambda)=\eta^{-1}\mathbb{E}_x
\operatorname{Var}_{\widehat\pi_\lambda(\cdot\mid x)}(\widehat r_2)
\]
can be much smaller than $B^2/\eta$. This explains why larger steps may
be useful, but does not itself prove convergence of the heuristic.
Throughout, we track the instantaneous multiplier $\lambda_t$ and the
time-averaged value $\bar\lambda_T=\frac1T\sum_{t=0}^{T-1}\lambda_t$,
and use the latter to define the final policy;
Figure~\ref{fig:constraint violation} reports the estimated constraint
residual over iterations. All policies and gradients in this experiment
use the estimated rewards. The existing runs use this averaged
multiplier; the last-iterate and conservative-bisection guarantees
concern the respective algorithms specified in the theory.

The Gibbs update reallocates probability according to the estimated
combined reward. With two candidate actions,
\[
\log\frac{\widehat\pi_\lambda(0\mid x)}
{\widehat\pi_\lambda(1\mid x)}
-\log\frac{\pi_0(0\mid x)}{\pi_0(1\mid x)}
=\frac{\widehat\Delta_1(x)+\lambda\widehat\Delta_2(x)}{\eta},
\]
where $\widehat\Delta_k(x)=\widehat r_k(x,0)-\widehat r_k(x,1)$.
The probability-shift plot is therefore a diagnostic of the Gibbs update,
not independent evidence that the reward model is accurate.
Table~\ref{tab:performance} reports learned reward scores and the positive
part of their average constraint shortfall, where
$\widehat{\mathbb E}$ averages over the indicated data split and the
response-selection policy. These quantities do not measure true latent
welfare or establish population-level feasibility.

To demonstrate the practical behavior of the algorithm when using the basic step size \(\alpha=\frac{\eta}{B^2}\), we include supplementary simulations on synthetic data in Appendix~\ref{app:sim}.

\begin{figure*}[h]
    \centering
    \begin{subfigure}[b]{0.35\textwidth}
        \centering
        \includegraphics[width=0.95\textwidth]{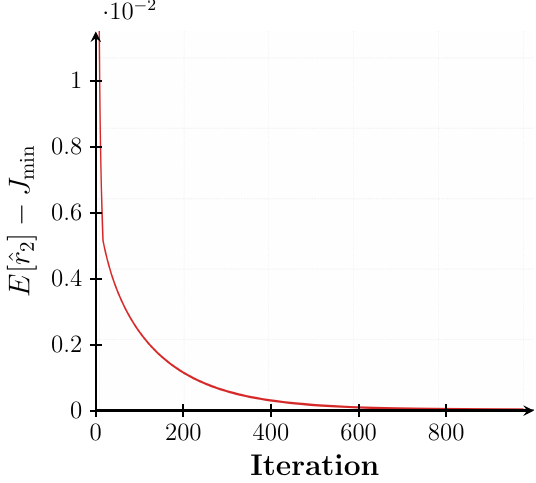}
        \caption{Estimated constraint residual}
        \label{fig:constraint violation}
    \end{subfigure}
    \hfill
    \begin{subfigure}[b]{0.6\textwidth}
        \centering
        \includegraphics[width=0.95\textwidth]{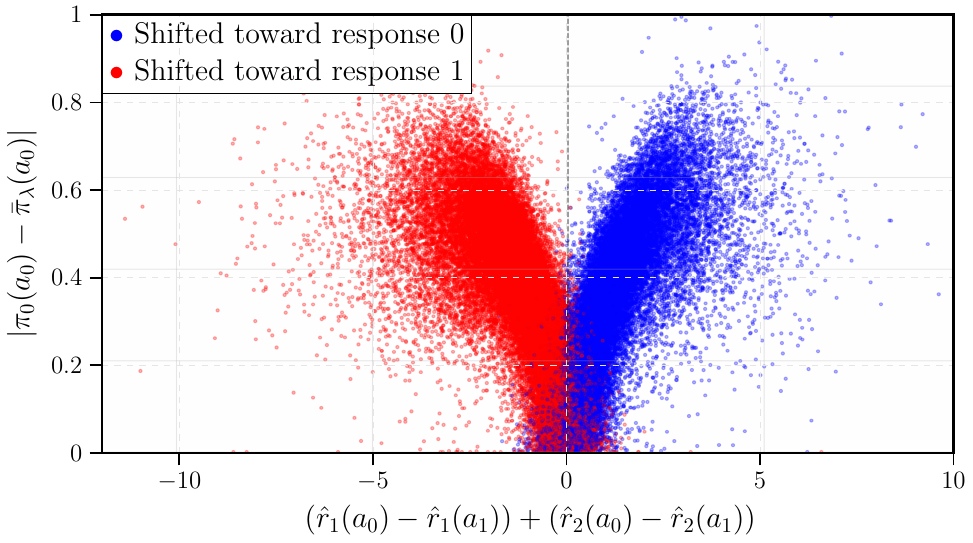}
        \caption{Policy shift vs. reward difference}
        \label{fig:policy shift vs rewards}
    \end{subfigure}
    \caption{}
    \label{fig:results}
\end{figure*}

\begin{figure}[h]
    \centering
    \begin{subfigure}[b]{0.5\textwidth}
        \centering
        \includegraphics[width=1\textwidth]{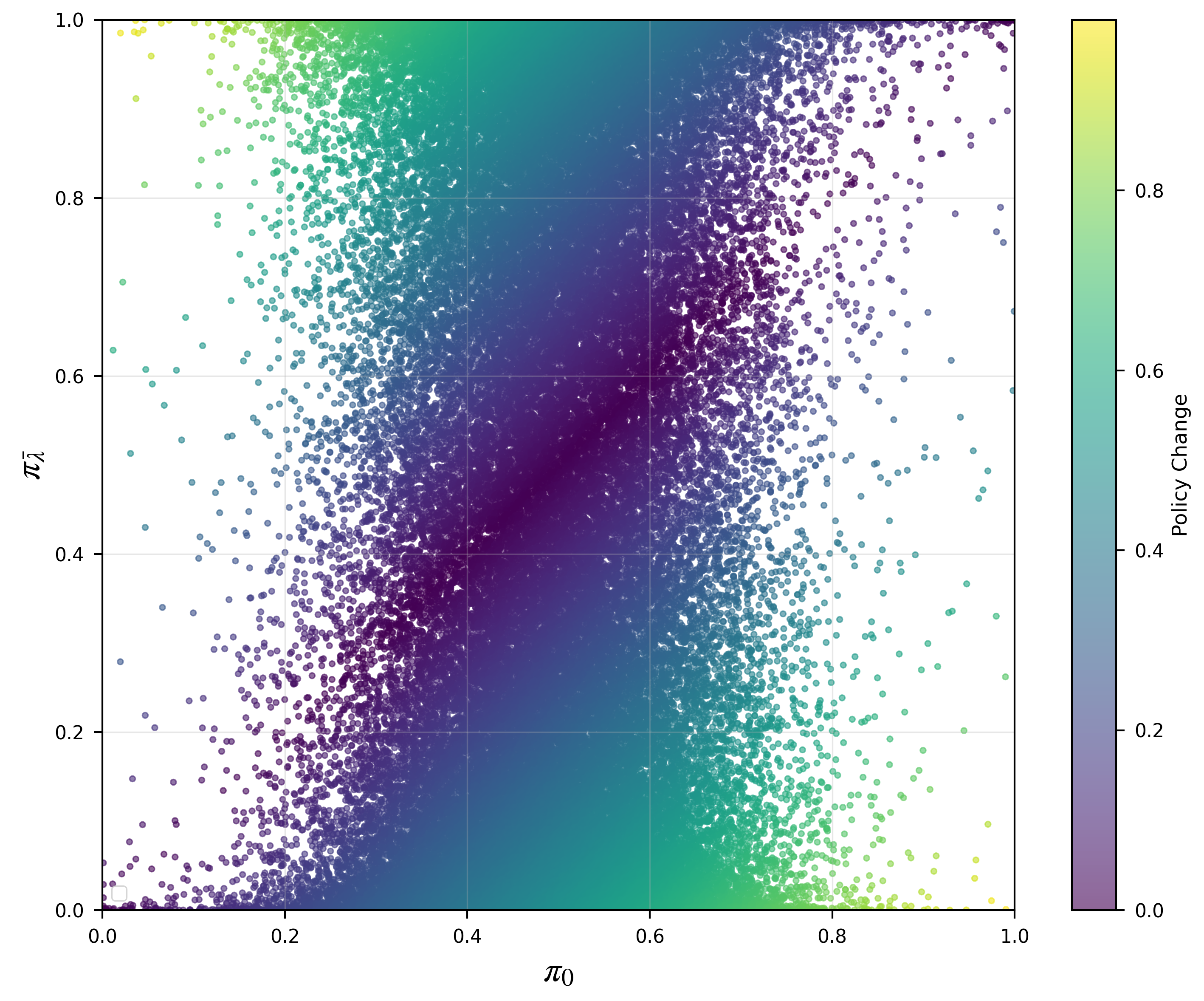}
    \end{subfigure}
    \caption{Policy shift}
    \label{fig:policy shift}
\end{figure}

\begin{table}[h]
\centering
\caption{Learned reward scores in the two-candidate PKU-SafeRLHF experiment.}
\label{tab:performance}
\begin{tabular}{lccc}
\toprule
& \multicolumn{2}{c}{\textbf{Training (74k)}} & \textbf{Test (7k)} \\
\cmidrule(lr){2-3} \cmidrule(lr){4-4}
\textbf{Metric} & $\pi_0$ & $\pi_{\bar{\lambda}}$ & $\pi_{\bar{\lambda}}$ \\
\midrule
Safety score $\widehat{\mathbb{E}}[\widehat r_1]$ & -1.8997 & -1.6034 & -1.6044 \\
Helpfulness score $\widehat{\mathbb{E}}[\widehat r_2]$ & -0.8867 & -0.6479 & -0.6287 \\
Estimated violation & 0.2394 & 0.0006 & 0.0000 \\
\midrule
\multicolumn{4}{l}{\textbf{Parameters:} $J_{\min} = -0.6473$, $\eta = 0.3$, $T = 1000$} \\
\bottomrule
\end{tabular}
\end{table}

\section{Extensions}\label{sec:extension}
\subsection{Multiple Constrained Oracles}
Let $h=(r_2^*,\ldots,r_{m+1}^*)^\top$,
$b=(J_{2,\min},\ldots,J_{m+1,\min})^\top$, and index the
multiplier vector as $\lambda=(\lambda_2,\ldots,\lambda_{m+1})$.
For $\lambda\ge0$,
\[
g(\lambda)=\eta\mathbb{E}_x\log\sum_a\pi_0(a\mid x)
\exp\big((r_1^*(x,a)+\lambda^\top h(x,a))/\eta\big)-\lambda^\top b.
\]
The associated Gibbs policy maximizes the Lagrangian, and
\[
\nabla g(\lambda)=\mathbb{E}_{\pi_\lambda^*}[h]-b,
\qquad
\nabla^2g(\lambda)=\frac1\eta\mathbb{E}_x
\operatorname{Cov}_{\pi_\lambda^*(\cdot\mid x)}(h).
\]
Writing $\Theta_C=[\theta_2^*,\ldots,\theta_{m+1}^*]$,
a valid Euclidean smoothness constant is
$L_m=\|\Theta_C\|_{\rm op}^2/\eta\le mB^2/\eta$.
On an event with uniform errors
$\|\widehat r_k-r_k^*\|_\infty\le\varepsilon_k$,
\[
|\widehat g(\lambda)-g(\lambda)|
\le\varepsilon_1+\sum_{k=2}^{m+1}\lambda_k\varepsilon_k.
\]
Simultaneous MLE confidence follows by a union bound over the
$m+1$ oracles. The confidence logarithm becomes
$\log((m+1)/\delta)$, with the probability budget adjusted
if a separate covariance event is used.

A sufficient condition for strong duality is one jointly strictly
feasible policy:
$\mathbb{E}_{\widetilde\pi}r_k^*\ge J_{k,\min}+\rho_k$
for all $k=2,\ldots,m+1$, with every $\rho_k>0$.
Its optimal multipliers satisfy
\[
\sum_{k=2}^{m+1}\rho_k\lambda_k^*
\le g(0)-J(\widetilde\pi).
\]
Strong convexity additionally requires that no nonzero linear
combination of the constrained rewards is action-independent
almost surely. Under feature identifiability, this is equivalent
to $\Theta_C$ having full column rank, and therefore requires $m\le d$.

The learned-policy certificate extends componentwise, with
$c_k(\pi)\coloneqq J_{k,\min}-\mathbb{E}_\pi r_k^*$ the violation
of constraint $k$. For threshold increments $\tau_k\ge0$ and
empirical residuals
$H_k=\mathbb{E}_{\widehat\pi_\lambda}\widehat r_k
-J_{k,\min}-\tau_k$,
\begin{align*}
[c_k(\widehat\pi_\lambda)]_+
&\le[\varepsilon_k-\tau_k-H_k]_+,\\
J(\pi^*)-J(\widehat\pi_\lambda)
&\le A_N+\sum_{k=2}^{m+1}
\lambda_k(\varepsilon_k+\tau_k+H_k).
\end{align*}
The proof is the same Gibbs variational identity as in
Lemma~\ref{lem:policy-certificate}.
These statements do not imply the scalar last-iterate monotonicity
or bisection rates in dimension $m>1$, a multivariate optimization
analysis is needed for corresponding iteration bounds.

\subsection{General Divergence}
Our dual reduction with closed-form policy extends to any $f$-divergence where the regularized objective admits a closed form KKT characterization. Specifically we have the following proposition.
\begin{proposition}
\label{prop:f-div}
Let $f:[0,\infty)\to\mathbb{R}\cup\{+\infty\}$ be proper,
lower semicontinuous and strictly convex, finite and continuously
differentiable on $(0,\infty)$, with $f(1)=0$. Define
\[
\ell_f=\lim_{t\downarrow0}f'(t),\qquad
u_f=\lim_{t\uparrow\infty}f'(t),
\]
and, for $s<u_f$,
\[
\psi_f(s)=
\begin{cases}
0,&s\le\ell_f\text{ and }\ell_f>-\infty,\\
(f')^{-1}(s),&\ell_f<s<u_f.
\end{cases}
\]
For $r_\lambda(x,a)=r_1^*(x,a)+\lambda r_2^*(x,a)$ and a
full-support reference $\pi_0$, the unique optimizer under
$D_f(\pi\|\pi_0)=\sum_a\pi_0(a\mid x)
f(\pi(a\mid x)/\pi_0(a\mid x))$ is
\[
\pi_\lambda^f(a\mid x)=\pi_0(a\mid x)
\psi_f((r_\lambda(x,a)-\tau_x)/\eta),
\]
where $\tau_x$ is uniquely determined by normalization.
If $u_f<\infty$, it satisfies
$\tau_x>\max_a r_\lambda(x,a)-\eta u_f$.
\end{proposition}

For Pearson's divergence $f(t)=(t-1)^2$,
\[
\pi_\lambda^f(a\mid x)=\pi_0(a\mid x)
\left[1+\frac{r_\lambda(x,a)-\tau_x}{2\eta}\right]_+.
\]
For the power-divergence generator
\[
f_\alpha(t)=\frac{t^\alpha-\alpha t+\alpha-1}
{\alpha(\alpha-1)},\qquad\alpha>0,\quad\alpha\ne1,
\]
the optimizer is
\[
\pi_\lambda^f(a\mid x)=\pi_0(a\mid x)
\begin{cases}
\left[1+\frac{(\alpha-1)(r_\lambda(x,a)-\tau_x)}\eta
\right]_+^{1/(\alpha-1)},&\alpha>1,\\[1mm]
\left[1+\frac{(\alpha-1)(r_\lambda(x,a)-\tau_x)}\eta
\right]^{1/(\alpha-1)},&0<\alpha<1.
\end{cases}
\]
For $0<\alpha<1$, the base must be strictly positive for every
action, equivalently
$\tau_x>\max_a r_\lambda(x,a)-\eta/(1-\alpha)$.
The limit $\alpha\to1$ recovers KL regularization, which is the
focus of the main text.

The dual is convex and differentiable because the policy
maximizer is unique. Quantitative smoothness and strong convexity
require additional assumptions. If $f$ is $\kappa$-strongly convex
on the attainable ratio intervals $[0,1/\pi_0(a\mid x)]$ (for
instance, $f''\ge\kappa>0$ on their interiors), a valid scalar
smoothness bound is
\[
L_f\le\frac{\mathbb{E}_x
\operatorname{Var}_{\pi_0(\cdot\mid x)}(r_2^*)}{\eta\kappa}
\le\frac{B^2}{\eta\kappa}.
\]
Sparse optimizers can make the dual affine on nontrivial intervals;
therefore the KL strong-convexity conclusion does not automatically
extend. The KL-based quadratic learned-policy certificate above
also requires its own replacement for a general divergence.

\section{Conclusion and Future Work}
We proposed and analyzed a dual-only approach for offline constrained RLHF with multiple preference oracles: after MLE reward estimation from pairwise data, the KL-regularized primal admits a closed-form Gibbs policy and learning reduces to convex optimization in the dual. We provided finite-sample bounds in this setting, separating statistical error ($\approx O(\sqrt{d/N})$, driven by coverage) from optimization error ($\approx O(1/T)$), and gave data-dependent certificates for the dual radius and constraint satisfaction. Empirically, synthetic simulations and experiments on PKU-SafeRLHF align with the theory: the learned Gibbs policy achieves near-zero constraint violation on test data while improving target-group reward. As future work, we aim to extend our results to richer preference models, exploit cross-oracle dependence via joint estimation, and develop online constrained RLHF.

\bibliographystyle{plainnat}
\bibliography{main}
\newpage
\appendix
\section{Synthetic Experiments}\label{app:sim}
We simulate in a finite prompt-action environment where features are drawn as $\phi(x,a) \sim \mathcal{N}(0,I_d)$ and normalized to unit $L_2$-norm. Ground-truth parameters $\theta^*_1$ and $\theta^*_2$ are independently sampled from $\mathcal{N}(0,I_d)$ and normalized. We set $\theta_0 = w \theta^*_1 + (1-w)\theta^*_2$ and define $\pi_0(a|x) \propto \exp(\frac{1}{\eta_0}\left<\theta_0,\phi(x,a) \right>).$ Here, $w$ controls behavioral bias and $\eta_0$ controls coverage. We use $5$ seeds, and over each random seed we generate a dataset $\mathcal{D}_{N_{\max}}$ with $|\mathcal{D}_{N_{\max}}| =3,000$ samples by repeating: (1) sample $x$ uniformly from $\mathcal{X}$, (2) sample $a,a' \overset{i.i.d}{\sim} \pi_0(\cdot | x)$, and (3) draw Bradley-Terry preferences $y_1,y_2$ using $\theta^*_1,\theta^*_2$.

For each random seed, we evaluate convergence by measuring performance on the first $N$ samples of the dataset, increasing $N$ from $0$ to $N_{\max}=3000$ in increments of $300$. 
In our simulations, we set $|\mathcal{X}|=100$ and $|\mathcal{A}|=10$ to reduce computational overhead.
For each oracle $k \in \{1,2\}$, we estimate $\widehat{\theta}_k$ using a regularized Bradley–Terry MLE (for stability at small $N$), applied to the pairwise feature differences $\phi(x,a)-\phi(x,a')$, optimized with L-BFGS.
We minimize the empirical dual using projected gradient descent as in
Algorithm~\ref{alg:pgd-dual}. For each fitted reward model, define
$\widehat B_2=\max_{x,a}|\widehat r_2(x,a)|$ and use the empirical
smoothness bound $\widehat L=\widehat B_2^2/\eta$ with $\eta=.05$;
we take $\alpha=1/\widehat L$ when $\widehat L>0$ and handle the
constant-reward case separately. 

To generate an active yet feasible constraint level, we calibrate $J_{\min}$ once per configuration using the ground-truth reward and an expanded dataset of $10{,}000$ samples to ensure stability:
$
    J_{\min} = E_0 + \text{frac}\cdot (E_{\text{hi}} - E_0),
$
where $E_0$ is the constraint expectation at $\lambda = 0$ and $E_{\text{hi}}$ is the expectation at $\lambda_{\text{hi}} = 5$. Each figure reports averages over random seeds, using the same per-seed dataset prefixes across parameter configurations. When varying $w$, we regenerate $\pi_0$ and the per-seed datasets, and recalibrate $J_{\min}$ accordingly; because $J_{\min}$ is recalibrated whenever $w$ changes, these comparisons change the optimization problem as well as the reference policy, and do not isolate the effect of coverage alone. For comparison, we approximate $\lambda^*$ via gradient descent using ground-truth rewards with increased iterations for accuracy, and then recover the corresponding optimal Gibbs policy as in our analysis.

Our simulations illustrate the predicted qualitative trends: Figure~\ref{fig:combo} shows convergence across three values of $w$, with both primal suboptimality and constraint violation decreasing as $N$ increases. Shaded regions indicate confidence intervals over $5$ random seeds, illustrating the consistency of our approach. As $w$ increases from $0.3$ (which biases the generating policy toward oracle $2$) to $w= 0.9$ (which biases toward oracle $1$), we observe a higher initial constraint violation. In each setting, however, we observe convergence to near-zero violation and sub-optimality. We plot the signed residual $c(\pi)=J_{\min}-\mathbb{E}_\pi r_2^*$;
negative values indicate feasibility slack. The violation bounded in
our theorem is $[c(\pi)]_+$. The signed objective gap can also be
negative for an infeasible policy.

\begin{figure*}[h]
  \centering
  \begin{subfigure}[t]{0.33\textwidth}
    \centering
    \includegraphics[width=\linewidth]{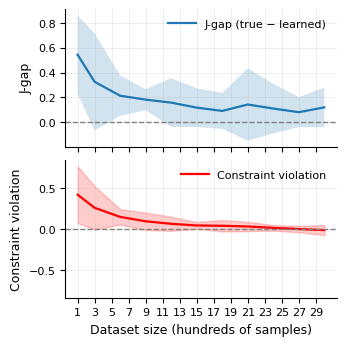}
    \caption{$w = 0.3$}
    \label{fig:combo:a}
  \end{subfigure}\hfill
  \begin{subfigure}[t]{0.33\textwidth}
    \centering
    \includegraphics[width=\linewidth]{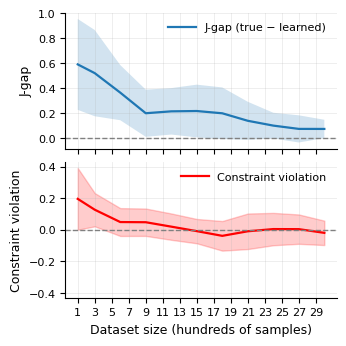}
    \caption{$w = 0.6$}
    \label{fig:combo:b}
  \end{subfigure}
  \begin{subfigure}[t]{0.33\textwidth}
    \centering
    \includegraphics[width=\linewidth]{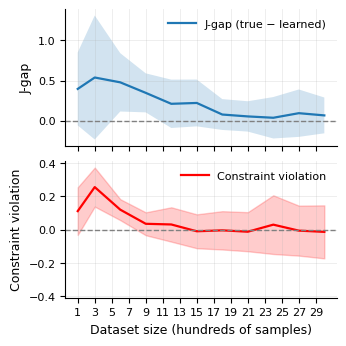}
    \caption{$w = 0.9$}
    \label{fig:combo:c}
  \end{subfigure}\hfill
  \caption{Performance vs.\ dataset size ($N$) with $T=1000$ over three settings of $w$.
  \textbf{Top:} Primal objective sub-optimality.
  \textbf{Bottom:} Constraint violation.}
  \label{fig:combo}
\end{figure*}

\section{Proofs}
\subsection{Proof of Corollary~\ref{cor:slaterData}}
\begin{proof}
Let $\psi(\pi)=\mathbb{E}_\pi\phi(x,a)$. On the common confidence
event, for every policy $\pi$,
\[
|\mathbb{E}_\pi(\widehat r_2-r_2^*)|
=|\langle\widehat\theta_2-\theta_2^*,\psi(\pi)\rangle|
\le\beta_N\|\psi(\pi)\|_{V_N^{-1}}.
\]
This remains valid for the data-dependent greedy policy. Substituting
$\widetilde\pi$ proves both assertions.
\end{proof}

\subsection{Proof of Lemma~\ref{lem:gradLips}}
\begin{proof}
Differentiating the finite log-partition sum gives
$g'(\lambda)=\mathbb{E}_{\pi_\lambda^*}r_2^*-J_{\min}$ and
the stated variance formula. A random variable in an interval of
length $s$ has variance at most $s^2/4$. Applying this conditionally
on $x$ proves the bound. The estimated-reward case is identical.
\end{proof}

\subsection{Proof of Lemma~\ref{dual concentrations}}
\begin{proof}
For every policy, the difference between the two Lagrangians is
bounded in absolute value by $\varepsilon_1+\lambda\varepsilon_2$.
Taking maxima proves the value bound. For the derivative, add and
subtract $\mathbb{E}_{\widehat\pi_\lambda}r_2^*$.
The direct reward-estimation term is at most $\varepsilon_2$.
For the other term, interpolate between natural parameters
$(\theta_1^*+\lambda\theta_2^*)/\eta$ and
$(\widehat\theta_1+\lambda\widehat\theta_2)/\eta$.
The Jacobian of the feature mean is its covariance matrix, whose
operator norm is at most $1$ since $\|\phi\|_2\le1$.
Multiplication by $\|\theta_2^*\|_2\le B$ gives the claimed
$B(e_1+\lambda e_2)/\eta$ bound. The same term is also at most
$2B$. All statements are consequences of the common
confidence event.
\end{proof}

\subsection{Proof of Lemma~\ref{lem:norm-change}}
\begin{proof}
    The independent matrices $X_i=\Delta_i\Delta_i^\top$ satisfy
    $0\preceq X_i\preceq4I$ and
    $\lambda_{\min}(\sum_i\mathbb{E} X_i)=Ns$.
    The matrix Chernoff lower-tail inequality yields
    $\mathbb{P}\{\lambda_{\min}(\sum_iX_i)<(1-a)Ns\}
    \le d\exp(-Nsa^2/8)$.
    Set $a=1/2$, divide by $N$, and add $\nu I$.
    A union bound combines the covariance event with the two MLE events.
\end{proof}
\subsection{Proof of Proposition~\ref{prop:strict-convexity}}
\begin{proof}
    Use the variance formula from Lemma~\ref{lem:gradLips}. Full reference
    support implies full Gibbs support at every finite $\lambda$.
    Nonconstancy on a positive-probability set of prompts makes the
    expected variance positive. Continuity and compactness then give
    a positive minimum. For the quantitative bound, on $[0,R]$,
    \[
    \frac{\pi_\lambda^*(a\mid x)}{\pi_0(a\mid x)}
    \ge e^{-2B(1+R)/\eta}.
    \]
    Using $\operatorname{Var}_p(f)=\min_c\mathbb{E}_p(f-c)^2$ gives
    $\operatorname{Var}_{\pi_\lambda^*}(r_2^*)
    \ge e^{-2B(1+R)/\eta}\operatorname{Var}_{\pi_0}(r_2^*)$.
    Finally, independent reference responses imply
    $\mathbb{E}_x\operatorname{Var}_{\pi_0}(r_2^*)
    =(\theta_2^*)^\top\Sigma_\infty\theta_2^*/2$.
\end{proof}
\subsection{Proof of Theorem~\ref{thm:lambdaBound}}
\begin{proof}
    Strong duality and strict concavity of the policy objective imply
    $J(\pi^*)=g(\lambda^*)$ and
    $\pi^*=\pi^*_{\lambda^*}$. Also,
    \[
    g(0)\ge g(\lambda^*)
    \ge\mathcal L(\widetilde\pi,\lambda^*)
    \ge J(\widetilde\pi)+\lambda^*\rho.
    \]
    This proves the first radius bound; $g(0)\le B$ proves the second.
    The confidence event gives
    $g(0)\le\widehat g(0)+\varepsilon_1$,
    $J(\widetilde\pi)\ge\widehat J(\widetilde\pi)-\varepsilon_1$,
    and true slack at least
    $\widehat w_2(\widetilde\pi)-J_{\min}-\varepsilon_2$.
    Substitution proves the data-dependent true radius.
    The same Slater argument applied deterministically to the empirical
    problem with threshold $J_{\min}+\tau$ proves $R_\tau$.
\end{proof}

\subsection{Proof of Lemma~\ref{lem:policy-certificate}}
\begin{proof}
Fix $\lambda\ge0$, put $q=\widehat\pi_\lambda$, and define
$e_1(x,a)=r_1^*(x,a)-\widehat r_1(x,a)$.
The Gibbs variational identity gives, for every policy $\pi$,
\begin{align*}
\widehat J(q)+\lambda\widehat w_2(q)
 -\widehat J(\pi)-\lambda\widehat w_2(\pi)
=\eta\mathbb{E}_x
 D_{\mathrm{KL}}(\pi(\cdot\mid x)\|q(\cdot\mid x)).
\end{align*}
Apply this identity to $\pi^*$ and rearrange:
\begin{align*}
J(\pi^*)-J(q)
=\mathbb{E}_{\pi^*}e_1-\mathbb{E}_qe_1
 -\eta\mathbb{E}_x D_{\mathrm{KL}}(\pi^*\|q)
+\lambda\big(\widehat w_2(q)-\widehat w_2(\pi^*)\big).
\end{align*}
Feasibility of $\pi^*$ and reward confidence imply
$\widehat w_2(\pi^*)\ge J_{\min}-\varepsilon_2$.
Thus the multiplier term is at most
$\lambda(\varepsilon_2+\tau+h_\tau(\lambda))$.
For each prompt, the variational formula for relative entropy gives
\begin{align*}
&\mathbb{E}_{\pi^*(\cdot\mid x)}e_1
 -\mathbb{E}_{q(\cdot\mid x)}e_1
 -\eta D_{\mathrm{KL}}(\pi^*(\cdot\mid x)\|q(\cdot\mid x))\\
&\quad\le\eta\log\mathbb{E}_{q(\cdot\mid x)}e^{e_1/\eta}
 -\mathbb{E}_{q(\cdot\mid x)}e_1
 \le\frac{\varepsilon_1^2}{2\eta},
\end{align*}
where the last step is Hoeffding's lemma and
$e_1(x,a)\in[-\varepsilon_1,\varepsilon_1]$.
The same expression is at most $2\varepsilon_1$ by nonnegativity
of relative entropy. Average over prompts and take the smaller bound.
Finally,
\[
\mathbb{E}_qr_2^*
\ge\widehat w_2(q)-\varepsilon_2
=J_{\min}+\tau+h_\tau(\lambda)-\varepsilon_2,
\]
which proves the feasibility statement.
\end{proof}

For a sharper result, we may define
\[
D_x=\max_{a,a'}\|\phi(x,a)-\phi(x,a')\|_{V_N^{-1}},
\qquad
s_x=\min\{\beta_ND_x,\,2\varepsilon_1\}.
\]
Then the same certificate holds with $A_N$ replaced by
$A_N^{\rm span}=\mathbb{E}_x\min\{s_x,\,s_x^2/(8\eta)\}\le A_N$
This uses the reward-error span, so prompt-dependent reward offsets
do not worsen the target objective term. The tightest bound based
only on these conditional spans is
\[
A_N^{\rm mgf}=\eta\,\mathbb{E}_xK(s_x/\eta),
\qquad
K(t)=\log\frac{e^t-1}{t}-1+\frac{t}{e^t-1},
\quad K(0)=0,
\]
which is no larger than $A_N^{\rm span}$. Indeed, shifting a variable
of span $s$ to $[0,s]$, bounding its exponential by the chord through
the endpoints, and maximizing $\log(1+p(e^t-1))-pt$ over $p\in[0,1]$,
attained at $p=1/t-1/(e^t-1)$, gives the displayed formula.

\subsection{Proof of Theorem~\ref{thm:main}}
\begin{proof}
    If $\widehat g'(0)\ge0$, zero minimizes the empirical dual,
    and the certificate lemma with $\lambda=\tau=0$ proves the claims.
    Otherwise, let $a\in(0,R]$ be an empirical minimizer.
    It exists by empirical Slater, and $\widehat g'(a)=0$.
    For $0\le\lambda_t\le a$, convexity and smoothness imply
    \[
    0\le-\widehat g'(\lambda_t)\le\widehat L(a-\lambda_t).
    \]
    Thus projection does not change the update, and
    $0\le\lambda_t\le\lambda_{t+1}\le a$.
    Set $r_t=-\widehat g'(\lambda_t)$. This sequence is nonnegative
    and nonincreasing. Consequently,
    \[
    \lambda_T=\frac1{\widehat L}\sum_{t=0}^{T-1}r_t
    \ge\frac{Tr_T}{\widehat L},
    \qquad a-\lambda_T\ge\frac{r_T}{\widehat L}.
    \]
    Adding gives $r_T\le\widehat L a/(T+1)$.
    For $u\in[\lambda_T,a]$, we have
    $-\widehat g'(u)\le\min\{r_T,\widehat L(a-u)\}$, so
    \begin{align*}
    \widehat g(\lambda_T)-\widehat g(a)
    &\le r_T(a-\lambda_T)-\frac{r_T^2}{2\widehat L}\\
    &\le ar_T-\frac{T+1/2}{\widehat L}r_T^2
    \le\frac{\widehat L a^2}{4T+2}.
    \end{align*}
    Apply Lemma~\ref{lem:policy-certificate} with $\tau=0$ and
    $h_0(\lambda_T)=-r_T$. In the objective bound, discard the
    nonpositive term $-\lambda_T r_T$, and use
    $\lambda_T\le a\le R$.
\end{proof}

\subsection{Proof of Theorem~\ref{thm:certified}}
\begin{proof}
    If $h(0)\ge0$, apply Lemma~\ref{lem:policy-certificate} with
    $\lambda=0$ and $\tau=\varepsilon_2$. The violation bound
    vanishes and the objective bound reduces to $A_N$.
    Otherwise $h(0)<0$. Since $\widehat g$ is convex, $h$ is
    nondecreasing, and empirical Slater together with
    Theorem~\ref{thm:lambdaBound} (with $\tau=\varepsilon_2$)
    ensures a minimizer $a\in(0,R]$ of
    $\widehat g(\lambda)-\lambda\varepsilon_2$ with $h(a)=0$.
    Hence $h(R)\ge0$ and the initial bracket $[0,R]$ is valid.
    Bisection preserves $h(\ell_t)<0\le h(u_t)$ and keeps a root
    bracketed in an interval of length $R2^{-t}$. Since $h$ is
    nondecreasing and $\widehat L$-Lipschitz,
    \[
    0\le h(u_T)\le h(u_T)-h(\ell_T)
    \le\widehat L(u_T-\ell_T)\le\widehat L R2^{-T}.
    \]
    Applying Lemma~\ref{lem:policy-certificate} with
    $\tau=\varepsilon_2$ and $\lambda=u_T$ gives
    $[c(\widehat\pi_{u_T})]_+=0$ and
    $J(\pi^*)-J(\widehat\pi_{u_T})
    \le A_N+u_T(2\varepsilon_2+h(u_T))$.
    Substituting $u_T\le R$ and the bracket bound proves the
    displayed worst-case rate. Feasibility of $\widehat\pi_{u_T}$
    makes the objective gap nonnegative.
\end{proof}

\subsection{Proof of Proposition~\ref{prop:f-div}}
\begin{proof}
    Fix a prompt and abbreviate $q_a=\pi_0(a\mid x)>0$,
    $p_a=\pi(a\mid x)$, and $r_a=r_\lambda(x,a)$.
    The feasible simplex is compact, the objective is upper
    semicontinuous, and $p=q$ has finite objective, so a maximizer
    exists. Strict convexity of the divergence makes it unique.
    Writing $p_a=q_at_a$, the KKT conditions with normalization
    multiplier $\tau$ are
    \[
    t_a\in\arg\max_{t\ge0}\{((r_a-\tau)/\eta)t-f(t)\}.
    \]
    For arguments below $u_f$, this scalar optimizer is $\psi_f$.
    The normalization function
    \[
    S(\tau)=\sum_aq_a\psi_f((r_a-\tau)/\eta)
    \]
    is continuous and decreases strictly wherever it is positive.
    It tends to zero as $\tau\to\infty$ and to infinity as $\tau$
    approaches the lower endpoint of its domain. Hence there is a
    unique $\tau_x$ with $S(\tau_x)=1$. The resulting policy
    satisfies the KKT conditions and is optimal.
\end{proof}

\section{A Fisher--Bernstein Confidence Radius}
\label{app:fisher-bernstein}
The main-text confidence radius contains a universal numerical
constant $C$. When tight constants and numerical certificates
matter, the following refinement replaces $C$ by computable
quantities via a Fisher--Bernstein argument. This improves the
leading dependence on the curvature parameter in a sufficiently
large sample regime and does not require a full rank design. The
argument is self contained.
\begin{lemma}
\label{lem:explicit-MLE}
Condition on the comparison design and let
\[
r_N=\operatorname{rank}(\Sigma_N),\quad
h_N=\max_i\sqrt{\Delta_i^\top\Sigma_N^\dagger\Delta_i},\quad
q_N=r_N\log5+\log(12/\delta),
\]
where $\Sigma_N^\dagger$ is the Moore--Penrose pseudoinverse.
Let $\gamma_B=(2+e^{2B}+e^{-2B})^{-1}$ and
$c_B=4B/(1-e^{-4B})$, with its continuous value $c_0=1$.
Define
\begin{align*}
b_H&=\frac1{\gamma_B}\sqrt{\frac{2q_N}{N}},\\
b_F&=\frac{c_B}{\sqrt{\gamma_B}}
\left(2\sqrt{\frac{2q_N}{N}}+
\frac{4h_Nq_N}{3N\sqrt{\gamma_B}}\right),\\
\beta_N'
&=\sqrt{\min\{b_H^2,b_F^2\}+4\nu B^2}.
\end{align*}
For the two norm-constrained MLEs, simultaneously,
\[
\|\widehat\theta_k-\theta_k^*\|_{V_N}
\le\beta_N',\qquad k=1,2,
\]
with probability at least $1-2\delta/3$.
The result does not require full rank of $\Sigma_N$. If its rank
is zero, its seminorm term is zero. The additional covariance
event in Lemma~\ref{lem:norm-change} gives overall probability
$1-\delta$.
\end{lemma}
\begin{proof}
    Fix an oracle, let $\mathcal{Q}(\theta)=-\ell(\theta)/N$,
    $v=\widehat\theta-\theta^*$, and
    \[
    p_i=\sigma(\Delta_i^\top\theta^*),\quad
    F=\frac1N\sum_i p_i(1-p_i)\Delta_i\Delta_i^\top,\quad
    S=\nabla\mathcal{Q}(\theta^*)=\frac1N\sum_i(p_i-y_i)\Delta_i.
    \]
    The matrices $F$ and $\Sigma_N$ have the same range, and
    $F\succeq\gamma_B\Sigma_N$. The constrained first-order
    condition is $v^\top\nabla\mathcal{Q}(\widehat\theta)\le0$.
    The usual Hessian lower bound therefore gives
    \[
    \|v\|_{\Sigma_N}\le
    \gamma_B^{-1}\|S\|_{\Sigma_N^\dagger}.
    \]
    For a unit vector in the range, a projection of
    $\Sigma_N^{\dagger/2}S$ is a centered independent sum whose
    squared summand-range lengths sum to $1/N$. Hoeffding's
    inequality, a $1/2$-net of at most $5^{r_N}$ points, and the
    factor-two net norm bound yield $\|v\|_{\Sigma_N}\le b_H$
    outside an event of probability at most
    $2\cdot5^{r_N}e^{-q_N}$.
    For the second bound, $|(\log\sigma')'(z)|\le1$ and
    $|\Delta_i^\top v|\le4B$ imply
    \[
    v^\top[\nabla\mathcal{Q}(\widehat\theta)
    -\nabla\mathcal{Q}(\theta^*)]
    \ge c_B^{-1}\|v\|_F^2.
    \]
    The constrained first-order condition then gives
    $\|v\|_F\le c_B\|S\|_{F^\dagger}$.
    Conditionally on the design, $Z=F^{\dagger/2}S$ has mean zero
    and covariance $P_{\operatorname{range}(F)}/N$. Each summand
    of any unit projection is bounded in absolute value by
    $h_N/(N\sqrt{\gamma_B})$. Bernstein's inequality gives
    \[
    |u^\top Z|\le\sqrt{2q_N/N}
    +\frac{2h_Nq_N}{3N\sqrt{\gamma_B}}
    \]
    except with probability $2e^{-q_N}$. The same net argument
    and $F\succeq\gamma_B\Sigma_N$ give
    $\|v\|_{\Sigma_N}\le b_F$. A union bound over the two bounds
    and two oracles has failure probability at most
    $8\cdot5^{r_N}e^{-q_N}=2\delta/3$. Finally,
    $\|v\|_2\le2B$, so
    $\|v\|_{V_N}^2\le\min\{b_H^2,b_F^2\}+4\nu B^2$. All arguments
    take place on the common range when the design is singular.
\end{proof}
The leading Fisher term in $b_F$ has order
$O((1+B)\sqrt{(d+\log(1/\delta))/(\gamma_B N)})$, but the
$\gamma_B^{-1}h_Nq_N/N$ remainder must be retained. The improved
leading behavior is meaningful once $N$ is sufficiently large
relative to $h_N^2q_N/\gamma_B$. Taking the minimum with $b_H$
avoids worsening the Hoeffding bound.
\end{document}